\documentclass[11pt, english, submission, copyright, creativecommons]{eptcs}
\newcommand{\event}{ACT2019}

\input{header}

\title{Functorial Question Answering}
\author{
Giovanni de Felice$^\dagger$, Konstantinos Meichanetzidis$^{\dagger\star}$, Alexis Toumi$^\dagger$
\institute{$\dagger$ Department of Computer Science, University of Oxford.\\ $\star$ Cambridge Quantum Computing Ltd.}}

\def\titlerunning{Functorial Question Answering}
\def\authorrunning{G. de Felice, K. Meichanetzidis \& A. Toumi}

\begin{document}
\maketitle

\renewcommand{\tt}[1]{\mathtt{#1}}
\renewcommand{\bf}[1]{\mathbf{#1}}
\renewcommand{\sc}[1]{\textsc{#1}}

\begin{abstract}
  Distributional compositional (DisCo) models are functors that
  compute the meaning of a sentence from the meaning of its words.
  We show that DisCo models in the category of sets and relations
  correspond precisely to relational databases. As a consequence, we get
  complexity-theoretic reductions from semantics and entailment of a
  fragment of natural language to evaluation and containment of conjunctive
  queries, respectively. Finally, we define question answering as an
  NP-complete problem.
\end{abstract}

\section*{Introduction}
%!TEX root = discocat-act2019.tex
Pregroup grammar is a model of natural language syntax first introduced by
Lambek \cite{Lambek99, Lambek08}, which admits a simple formulation in category
theory. The distributional compositional (DisCo) models of Coecke
et al. \cite{ClarkEtAl08, ClarkEtAl10} give a semantics to pregroup grammars
as functors into the category of vector spaces and linear maps.
They compute the meaning of a sentence as the composition of the vectors
encoding the distributional semantics of its constituent words.

In this paper, we consider DisCo models in the category of finite sets and
relations. We show they correspond to models of pregroup grammars in
\emph{regular logic}, the fragment of first-order logic generated by relational
symbols, equality $(=)$, truth $(\top)$, conjunction $(\land)$ and existential
quantification $(\exists)$. Regular logic formulae play a foundational role in
the theory of relational databases, where they corresponds to
\emph{conjunctive queries}. Chandra and Merlin \cite{ChandraMerlin77}
showed that conjunctive query evaluation and containment
are logspace equivalent to graph homomorphism: they are $\tt{NP-complete}$.
Bonchi et al. \cite{BonchiEtAl18}, reformulated this in terms of the free
Cartesian bicategory $\bf{CB}(\Sigma)$ generated by a relational signature
$\Sigma$: arrows are conjunctive queries, structure-preserving functors
$K: \bf{CB}(\Sigma) \rightarrow \bf{Rel}$ are precisely relational databases.

Relational DisCo models are functors $F : \bf{G} \to \bf{Rel}$ from the
rigid monoidal category $\bf{G}$ generated by a pregroup grammar.
We show that they factorise as $F = K \circ L$ for
$K : \bf{CB}(\Sigma) \rightarrow \bf{Rel}$ a relational database and
a free model $L : \bf{G} \to \bf{CB}(\Sigma)$,
turning sentences into regular logic formulae.
As a corollary, the computational problems for natural language $\tt{Semantics}$
and $\tt{Entailment}$ reduce to conjunctive query evaluation and containment, respectively.  Building on previous work \cite{CoeckeEtAl18a}, we show how to
translate a natural language corpus into a relational database.
The resulting $\tt{QuestionAnswering}$ problem is $\tt{NP-complete}$.

\subsection*{Related works}
Logical pregroup semantics has been developed in a line of work by Preller
\cite{PrellerSadrzadeh11, Preller14, Preller14a}, however the corresponding
reasoning problems were undecidable. Entailment in distributional models was considered in \cite{SadrzadehEtAl18}, but sentences could only be compared
if they have the same grammatical structure.

\section{Lambek Pregroups and Free Rigid Categories}
%!TEX root = discocat-act2019.tex
Pregroup grammar is an algebraic model of natural language grammar introduced by Lambek \cite{Lambek99}, it is
weakly equivalent to context-free grammars \cite{BuszkowskiMoroz08}.
In this section, we give a definition of the $\tt{Parsing}$ problem in terms of the homsets of the free \emph{rigid monoidal category} generated by a pregroup dictionary.

Given a natural number $n \in \N$, we abuse notation and let $n = \set{ i \in \N \ \vert \ i < n }$.
Given sets $X$ and $Y$, $X + Y$ and $X \times Y$ denote the disjoint sum and the Cartesian product respectively.
Let $List(X) = \bigcup_{n \in \N} X^n$ be the free monoid with unit $\epsilon \in X^0$ the empty list and product denoted by concatenation.

\begin{definition}
  A pregroup is a strict monoidal category with unit $\epsilon$ and product denoted by concatenation, which is thin --- i.e. with at most one arrow (denoted $\leq$) between any two objects --- and where each object $t$ has left and right adjoints denoted $^\star t$ and $t^\star$, i.e.
  \begin{itemize}
    \item \makebox[10cm][l]{$t ({}^{\star}t) \s \leq \s \epsilon \s \leq \s ({}^{\star}t) t$} (left adjunction)
    \item \makebox[10cm][l]{$(t^\star) t \s \leq \s \epsilon \s \leq \s t (t^\star)$} (right adjunction)
  \end{itemize}
\end{definition}

A pregroup grammar is a tuple $G = (V, B, \Delta)$ where $V$ is a finite set called the \emph{vocabulary}, $B$ is a poset of \emph{basic types} and $\rel{\Delta}{V}{P(B)}$ is a finite set of pairs called the \emph{dictionary}, where $P(B)$ is the free pregroup generated by $B$ as defined in \cite{Lambek99}.
We write $\Delta(u) = \set{t(0) \dots t(n - 1) \ \vert \ t \in \prod_{i < n} \Delta(u(i))}$ for $u \in V^n$.

\begin{definition}\begin{problem}
  \problemtitle{$\tt{Grammaticality}$}
  \probleminput{$G = (V, B, \Delta), \quad u \in List(V), \quad s \in P(B)$}
  \problemoutput{$\exists \ t \in \Delta(u) \s \cdot \s t \leq s$}
\end{problem}\end{definition}

\begin{example}\label{example0}
    Take the basic types $B = \set{s, q, d, n, i, o}$ for sentence, question, determinant, noun, subject and object respectively, with $n \leq i$ and $n \leq o$.
    The dictionary $\Delta$ assigns $({}^\star i) s (o^\star)$ to transitive verbs and $({}^\star d) n$ to common nouns.
    The word ``who'' is assigned both $({}^\star n) n (s^\star) i$ and $q (s^\star) i$.
    From the pregroup axioms it follows that:
    $$
    q (s^\star) i \s ({}^\star i) s (o^\star) \s d \s ({}^\star d) n \s ({}^\star n) n (s^\star) i \s ({}^\star i) s (o^\star) \s n \s \leq \s q
    $$
    i.e. $u = \text{``Who influenced the philosopher who discovered calculus?''}$ is grammatical.
\end{example}

\begin{lemma}[Switching lemma \cite{Lambek99}]
  For all $t \leq s \in P(B)$ there is some $t' \in P(B)$ and a pair of reductions $t \leq t'$ and $t' \leq s$ with no expansions and no contractions respectively.
\end{lemma}

\begin{corollary}[\cite{BuszkowskiMoroz08}]
$\tt{Grammaticality} \in \tt{P}$
\end{corollary}

\begin{proof}
  The proof goes by translating pregroup grammars to context-free grammars.
\end{proof}

As thin categories, pregroups cannot distinguish between distinct parsings of the same phrase, e.g. ``men and (women who read)'' and ``(men and women) who read''.
This motivated Preller, Lambek \cite{PrellerLambek07} to introduce free compact 2-categories, capturing proof-relevance in pregroup grammars.
We will use compact 2-categories with one 0-cell, first introduced in Joyal, Street \cite{JoyalStreet88} from which we use the \emph{planar string diagram} notation.
We refer the reader to Selinger's survey \cite{selinger2010a} where they are called \emph{rigid monoidal categories}.

\begin{definition}\label{def-rigid}
  A (strict) monoidal category $(\bf{C}, \otimes, \epsilon)$ is rigid when each object $t \in Ob(\bf{C})$ has left and right adjoints ${}^\star t$ and $t^\star$ and two pairs of arrows $t \otimes {}^\star t \to \epsilon \to {}^\star t \otimes t$ and $t {}^\star \otimes t \to \epsilon \to t \otimes t^\star$
  depicted by cups and caps, subject to the following snake equations:
%!TEX root = ../discocat-act2019.tex
\begin{equation*}
  \begin{gathered}
  \begin{tikzpicture} [scale=0.8, decoration={markings, mark=at position 0.5 with {\arrow{<}}}]
    \draw (-1, 1)--(-1, 0);

    \draw  (-1, 0) to [out=-90, in=180] (0, -1);
    \node[scale=1] at (-0.5, 0) {$t$};
    \draw  (0, -1) to [out=0, in=-90] (1, 0);

    \node[scale=1] at (1.5, 0) {${}^\star t$};

    \draw  (1, 0) to [out=90, in=180] (2, 1);
    \node[scale=1] at (3.5, 0) {$t$};
    \draw  (2, 1) to [out=0, in=90] (3, 0);

    \draw  (3, 0)--(3, -1);
  \end{tikzpicture}
  \end{gathered}
  \quad
  =
  \quad
  \begin{gathered}
  \begin{tikzpicture} [scale=0.8, decoration={markings, mark=at position 0.5 with {\arrow{>}}}]
    \draw (0,-1) to (0,1);
    \node[scale=1] at (0.5, 0) {$t$};
  \end{tikzpicture}
  \end{gathered}
  \quad
  =
  \quad
  \begin{gathered}
  \begin{tikzpicture} [scale=0.8, decoration={markings, mark=at position 0.5 with {\arrow{<}}}]
    \draw (1, 1)--(1, 0);

    \draw  (1, 0) to [out=-90, in=0] (0, -1);
    \node[scale=1] at (-2.5, 0) {$t$};
    \draw  (0, -1) to [out=180, in=-90] (-1, 0);

    \node[scale=1] at (-0.5, 0) {$t^\star$};

    \draw  (-1, 0) to [out=90, in=0] (-2, 1);
    \node[scale=1] at (1.5, 0) {$t$};
    \draw  (-2, 1) to [out=180, in=90] (-3, 0);

    \draw  (-3, 0)--(-3, -1);
  \end{tikzpicture}
  \end{gathered}
\end{equation*}
  We say a strong monoidal functor is rigid when it sends cups to cups and caps to caps.
\end{definition}

Given a pregroup grammar $G = (V, B, \Delta)$, the dictionary $\Delta \sub V \times P(B)$ defines an \emph{autonomous signature} with generating objects $V + B$ and arrows $\set{w \to t}_{(w, t) \in \Delta}$.
We write $\bf{G}$ for the free rigid category that it generates, also called the \emph{lexical category} in \cite{Preller14}.
An arrow $r : u \to s$ for an utterance $u \in List(V)$ is a proof that $u$ is a grammatical sentence, i.e. a dictionnary entry for each word followed by a diagram which encodes the reduction.
Hence, the free rigid category $\bf{G}$ allows us to encode parsing as a function problem.

\begin{lemma}
  Given a pregroup grammar $G = (V, B, \Delta)$, an utterance $u \in List(V)$ and a type $s \in P(B)$, we have $(G, u, s) \in \tt{Grammaticality} \iff \exists \ r \in \bf{G}(u, s)$.
\end{lemma}

\begin{proof}
  This follows from the switching lemma for compact 2-categories as proved in \cite{PrellerLambek07}.
  Arrows $r \in \bf{G}(u, s)$ are of the form $r : u \to t \to s$ for some type $t \in \Delta(u)$ with $t \leq s$.
\end{proof}

\begin{definition}\begin{problem}
  \problemtitle{$\tt{Parsing}$}
  \probleminput{$G = (V, B, \Delta), \quad u \in List(V), \quad s \in P(B)$}
  \problemoutput{$r \in \bf{G}(u, s)$}
\end{problem}\end{definition}

\begin{proposition}[\cite{moroz2011}]
  $\tt{Parsing}$ is poly-time computable in the size of the basic types $B$, the dictionary $\rel{\Delta}{V}{P(B)}$ and the length of the inputs $(u, s) \in List(V) \times P(B)$.
\end{proposition}

\begin{proof}
  The parsing algorithm has time complexity $n^3$ in general, restricted cases of interest in linguistic applications may be parsed in linear time, see \cite{Preller07a}.
\end{proof}

Note that our definition of the lexical category $\bf{G}$ differs slightly from \cite{Preller14} in that we take not only basic types $b \in B$ but also words $w \in V$ as generating objects.
This allows us to capture both type assignment and reduction as a single arrow as well as to define the semantics of pregroup grammars as a functor, see definition~\ref{def-semantics} where we will also make use of the following lemma.

\begin{lemma}\label{get-rid-of-induced-steps}
  For any pregroup grammar $G = (V, B, \Delta)$ there is an equivalent grammar $G' = (V, B', \Delta')$ such that $\Delta \sub \Delta'$ and $B'$ is a discrete poset, i.e. $a \leq b \implies a = b$.
\end{lemma}

\begin{proof}
  For each $(w, \ t a t') \in \Delta$ with $a \leq b \in B$ we add $(w, \ t b t')$ as a dictionnary entry.
  This yields a dictionnary $\Delta'$ of size polynomial in $\size{B} \times \size{\Delta}$ and basic types $B'$ given by the underlying set of the poset $B$ such that $G$ and $G' = (V, B', \Delta')$ are equivalent.
\end{proof}

\begin{example}\label{example}
  The following planar diagram $r : u \to q$ corresponds to the parsing of example \ref{example0}, where we have we have omitted the types for readability.
  We keep our notation consistent with the literature by depicting dictionnary entries $w \to t$ as triangles labeled by $w \in V$ with output $t \in P(B)$.
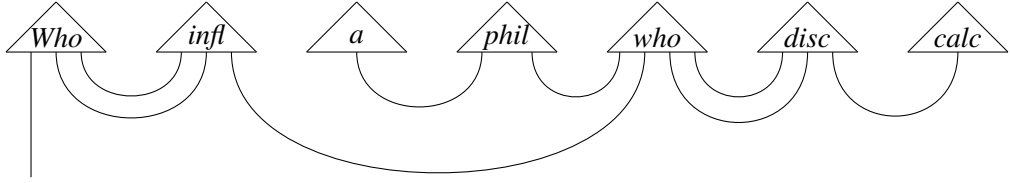
  $$\begin{tikzpicture}[scale=0.666]

  \node (90) at (-13, 0) {};
  \node (91) at (-12, 1) {};
  \node (92) at (-11, 0) {};
  \node (93) at (-10, 0) {};
  \node (94) at (-9, 1) {};
  \node (95) at (-8, 0) {};
  \node (96) at (-9.5, 0) {};
  \node (97) at (-9, 0) {};
  \node (98) at (-8.5, 0) {};
  \node (99) at (-12.5, 0) {};
  \node (910) at (-12, 0) {};
  \node (911) at (-11.5, 0) {};
  \node (912) at (-12.5, -2.5) {};

		\draw (91.center) to (90.center);
		\draw (91.center) to (92.center);
		\draw (92.center) to (90.center);
		\draw (93.center) to (94.center);
		\draw (94.center) to (95.center);
		\draw (95.center) to (93.center);
		\draw [bend right=90, looseness=1.50] (911.center) to (96.center);
		\draw [bend right=90, looseness=1.50] (910.center) to (97.center);
		\draw (99.center) to (912.center);

		\node (6) at (-7, 0) {};
		\node (7) at (-6, 1) {};
		\node (8) at (-5, 0) {};
		\node (9) at (-4, 0) {};
		\node (10) at (-3, 1) {};
		\node (11) at (-2, 0) {};
		\node (12) at (-1, 0) {};
		\node (13) at (0, 1) {};
		\node (14) at (1, 0) {};
		\node (15) at (2, 0) {};
		\node (16) at (3, 1) {};
		\node (17) at (4, 0) {};
		\node (18) at (5, 0) {};
		\node (19) at (6, 1) {};
		\node (20) at (7, 0) {};
		\node (21) at (-12, 0) {};
		\node (22) at (-9.5, 0) {};
		\node (23) at (-9, 0) {};
		\node (24) at (-8.5, 0) {};
		\node (25) at (-6.5, 0) {};
		\node (26) at (-5.5, 0) {};
		\node (27) at (-3, 0) {};
		\node (28) at (-0.75, 0) {};
		\node (29) at (-0.25, 0) {};
		\node (30) at (0.25, 0) {};
		\node (31) at (0.75, 0) {};
		\node (32) at (2.5, 0) {};
		\node (33) at (3, 0) {};
		\node (34) at (3.5, 0) {};
		\node (35) at (6, 0) {};

		\node (37) at (-12, 0.3) {Who};
		\node (38) at (-9, 0.3) {infl};
		\node (39) at (-6, 0.3) {a};
		\node (40) at (-3, 0.3) {phil};
		\node (41) at (0, 0.3) {who};
		\node (42) at (3, 0.3) {disc};
		\node (43) at (6, 0.3) {calc};

		\node (49) at (-8.5, -1) {};
		\node (47) at (3.5, -1) {};

		\draw (6.center) to (7.center);
		\draw (7.center) to (8.center);
		\draw (8.center) to (6.center);
		\draw (9.center) to (10.center);
		\draw (10.center) to (11.center);
		\draw (11.center) to (9.center);
		\draw (12.center) to (13.center);
		\draw (13.center) to (14.center);
		\draw (14.center) to (12.center);
		\draw (15.center) to (16.center);
		\draw (16.center) to (17.center);
		\draw (17.center) to (15.center);
		\draw (18.center) to (19.center);
		\draw (19.center) to (20.center);
		\draw (20.center) to (18.center);

		\draw [bend right=90, looseness=1.75] (31.center) to (32.center);
		\draw [bend left=90, looseness=1.75] (33.center) to (30.center);

		\draw [bend right=90, looseness=1.75] (34.center) to (35.center);

		\draw [bend left=90, looseness=1] (29.center) to (98.center);

      \node (0) at (-6, 0) {};
      \node (1) at (-6, 0.25) {};
      \node (2) at (-6, 0.25) {};
      \node (3) at (-5, 0) {};
      \node (4) at (-7, 0) {};
      \node (5) at (-6, 1) {};
      \node (6) at (-4, 0) {};
      \node (7) at (-3, 1) {};
      \node (8) at (-2, 0) {};
      \node (9) at (-3.5, 0) {};
      \node (10) at (-2.5, 0) {};
      \node (11) at (-1, 0) {};
      \node (12) at (0, 1) {};
      \node (13) at (1, 0) {};
      \node (14) at (-0.75, 0) {};

    		\draw [bend right=90, looseness=1.50] (0.center) to (9.center);
    		\draw [bend right=90, looseness=1.75] (10.center) to (14.center);
\end{tikzpicture}
$$
\end{example}

\section{Conjunctive Queries and Free Cartesian Bicategories}
%!TEX root = discocat-act2019.tex
Relational structures are a convenient mathematical abstraction of relational databases.
In this section, we review relational structures, their
relationship to conjunctive queries and their formulation
in terms of cartesian bicategories.

A \emph{relational signature} is a set of symbols $\Sigma$ equipped with a function $\tt{ar} : \Sigma \to \N$.
Given a finite set $U$, the set of \emph{$\Sigma$-structures} is $\cal{M}_\Sigma(U) = \set{K \sub \coprod_{R \in \Sigma} U^{\tt{ar}(R)}}$, i.e.
a $\Sigma$-structure $K$ gives an interpretation $K(R) \sub U^{\tt{ar}(R)}$ for every symbol $R \in \Sigma$.
Let $\cal{M}_\Sigma$ be the set of all finite $\Sigma$-structures and $U(K)$ the underlying universe of $K \in \cal{M}_\Sigma$.
Given two $\Sigma$-structures $K, K'$, a homomorphism $f : K \to K'$ is a function $f : U(K) \to U(K')$ such that $\forall \ R \in \Sigma \s \forall \ \vec{x} \in U^{\tt{ar}(R)} \s \cdot \s \vec{x} \in K(R) \implies f(\vec{x}) \in K'(R)$.

\begin{definition}
    \begin{problem}
    \problemtitle{$\tt{Homomorphism}$}
    \probleminput{$K, K' \in \cal{M}_\Sigma$}
    \problemoutput{$f : K \to K'$}
  \end{problem}
\end{definition}

\begin{proposition}\cite{GareyJohnson90}
  $\tt{Homomorphism}$ is $\tt{NP-complete}$.
\end{proposition}

\begin{proof}
  Membership may be shown to follow from Fagin's theorem: homomorphisms are defined by an existential second-order logic formula.
  Hardness follows by reduction from graph homomorphism: take $\Sigma = \set{\bullet}$ and $\tt{ar}(\bullet) = 2$ then a $\Sigma$-structure is a graph.
\end{proof}

Regular logic formulae are defined by the following context-free grammar:
$$
\phi \s ::= \s \top \s\vert\s x = x' \s\vert\s \phi \land \phi \s\vert\s \exists \ x \cdot \phi \s\vert\s R(\vec{x})
$$
where $x, x' \in \cal{X}$, $R \in \Sigma$ and $\vec{x} \in \cal{X}^{\tt{ar}(R)}$ for some countable set of variables $\cal{X}$.
We denote the variables of $\phi$ by $\tt{var}(\phi) \sub \cal{X}$, its free variables by $\tt{fv}(\phi) \sub \tt{var}(\phi)$ and its atomic formulae by $\tt{atoms}(\phi) \sub \coprod_{R \in \Sigma} \text{var}(\phi)^{\tt{ar}(R)}$.
\emph{Conjunctive queries} $\phi \in \cal{Q}_\Sigma$ are the prenex normal form $\phi = \exists\ x_0 \cdots \exists \ x_k \cdot \phi'$ of regular logic formulae, for the bound variables $\set{x_0, \dots, x_k} = \tt{var}(\phi) \setminus \tt{fv}(\phi)$ and $\phi' = \bigwedge \tt{atoms}(\phi)$.
Given a structure $K \in \cal{M}_\Sigma$, let $\tt{eval}(\phi, K) = \set{v \in U(K)^{\tt{fv}(\phi)} \s \vert \s  (K, v) \vDash \phi}$ where the satisfaction relation $(\vDash)$ is defined in the usual way.

\begin{definition}
    \begin{problem}
    \problemtitle{$\tt{Evaluation}$}
    \probleminput{$\phi \in \cal{Q}_\Sigma, \quad K \in \cal{M}_\Sigma$}
    \problemoutput{$\tt{eval}(\phi, K) \sub U(K)^{\tt{fv}(\phi)}$}
  \end{problem}
\end{definition}

\begin{definition}
    \begin{problem}
    \problemtitle{$\tt{Containment}$}
    \probleminput{$\phi, \phi' \in \cal{Q}_\Sigma$}
    \problemoutput{$\phi \sub \phi' \s\equiv\s \forall \ K \in \cal{M}_\Sigma \ \cdot \ \tt{eval}(\phi, K) \sub \tt{eval}(\phi', K)$}
  \end{problem}
\end{definition}

\begin{definition}
  Given a query $\phi \in \cal{Q}_\Sigma$, the canonical structure $CM(\phi) \in \cal{M}_\Sigma$ is given by $U(CM(\phi)) = \tt{var}(\phi)$ and
  $CM(\phi)(R) = \set{\vec{x} \in \tt{var}(\phi)^{\tt{ar}(R)} \ \vert \ R(\vec{x}) \in \tt{atoms}(\phi)}$
  for $R \in \Sigma$.
\end{definition}

\begin{theorem}[Chandra-Merlin \cite{ChandraMerlin77}]\label{chandra-merlin}
  The problems $\tt{Evaluation}$ and $\tt{Containment}$ are logspace equivalent to $\tt{Homomorphism}$, hence $\tt{NP-complete}$.
\end{theorem}

\begin{proof}
  Given a query $\phi \in \cal{Q}_\Sigma$ and a structure $K \in \cal{M}_\Sigma$, query evaluation $\tt{eval}(\phi, K)$ is given by the set of homomorphisms $CM(\phi) \to K$.
  Given $\phi, \phi' \in \cal{M}_\Sigma$, we have $\phi \sub \phi'$ iff there is a homomorphism $f : CM(\phi) \to CM(\phi')$ such that $f(\tt{fv}(\phi)) = \tt{fv}(\phi')$.
  Given a structure $K \in \cal{M}_\Sigma$, we construct $\phi \in \cal{Q}_\Sigma$ with $\tt{fv}(\phi) = \varnothing$, $\tt{var}(\phi) = U(K)$ and $\tt{atoms}(\phi) = K$.
\end{proof}

Bonchi, Seeber and Sobocinski \cite{BonchiEtAl18} introduced graphical conjunctive queries (GCQ), a graphical calculus where query containment is captured by the axioms of the \emph{free Cartesian bicategory} $\bf{CB}(\Sigma)$ generated by the relational signature $\Sigma$.

\begin{definition}[Carboni-Walters \cite{CarboniWalters87}]\label{def-CB}
  A Cartesian bicategory is a symmetric monoidal category enriched in partial orders such that:
  \begin{enumerate}
    \item every object is equipped with a special commutative Frobenius algebra,
    \item the monoid and comonoid structure of each Frobenius algebra are adjoint,
    \item every arrow is a lax comonoid homomorphism.
    \end{enumerate}
  A morphism of Cartesian bicategories is a strong monoidal functor which preserves the partial order, the monoid and the comonoid structure.
\end{definition}

\begin{theorem}(\cite[prop.~9,10]{BonchiEtAl18})\label{translation}
  Let $\bf{CB}(\Sigma)$ be the free Cartesian bicategory generated by one object and arrows $\set{ R : 0 \to \tt{ar}(R)}_{R \in \Sigma}$, see \cite[def.~21]{BonchiEtAl18}.
  There is a two-way semantics-preserving translation $\Theta : \cal{Q}_\Sigma \to \bf{CB}(\Sigma)$, $\Lambda : \bf{CB}(\Sigma) \to \cal{Q}_\Sigma$, i.e. for all $\phi, \phi' \in \cal{Q}_\Sigma$ we have $\phi \sub \phi' \iff \Theta(\phi) \leq \Theta(\phi')$,
  and for all arrows $d, d' \in \bf{CB}(\Sigma)$, $d \leq d' \iff \Lambda(d) \sub \Lambda(d')$.
\end{theorem}

\begin{proof}
  The translation is defined by induction from the syntax of regular logic formulae to that of GCQ diagrams and back. Note that given $\phi \in \cal{Q}_\Sigma$ with $\size{\tt{fv}(\phi)} = n$, we have $\Theta(\phi) \in \bf{CB}(\Sigma)(0, n)$ and similarly we have $\tt{fv}(\Lambda(d)) = m + n$ for $d \in \bf{CB}(\Sigma)(m, n)$, i.e. open wires correspond to free variables.
\end{proof}

The category $\bf{Rel}$ of sets and relations with Cartesian product as tensor, singleton as unit, the diagonal and its transpose as monoid and comonoid, subset as partial order, is a Cartesian bicategory.
Given a set $U$, the subcategory $\bf{Rel}\vert_U \injects \bf{Rel}$ with natural numbers $m, n \in \N$ as objects and relations $R \sub U^{m + n}$ as arrows is also a Cartesian bicategory. It is furthermore a PROP, i.e. a symmetric monoidal category with addition of natural numbers as tensor on objects.

\begin{proposition}(\cite[prop.~23]{BonchiEtAl18})\label{structure-functor-bijection}
  Structures $K \in \cal{M}_\Sigma(U)$ are in bijective correspondence with identity-on-objects morphisms of Cartesian bicategories $K : \bf{CB}(\Sigma) \to \bf{Rel}\vert_U$.
\end{proposition}

\begin{proof}
  By the universal property of the free Cartesian bicategory, an identity-on-objects morphism $K : \bf{CB}(\Sigma) \to \bf{Rel}\vert_U$ is uniquely determined by its image on generators $\set{K(R) \sub U^{\tt{ar}(R)}}_{R \in \Sigma}$: this is precisely the data for a $\Sigma$-structure.
\end{proof}

\begin{corollary}\label{isomorphisms}
  Let $[\bf{CB}(\Sigma), \bf{Rel}]$ denote the set of morphisms of Cartesian bicategories, there are bijective correspondences:
  $$\bf{CB}(\Sigma)(0,0)
  \s\stackrel{(1)}{\simeq}\s \set{\phi \in \cal{Q}_\Sigma \ \vert\ \tt{fv}(\phi) = \varnothing}
  \s\stackrel{(2)}{\simeq}\s \cal{M}_\Sigma
  \s\stackrel{(3)}{\simeq}\s [\bf{CB}(\Sigma), \bf{Rel}]$$
\end{corollary}

\begin{proof}
  (1) follows from theorem \ref{translation}, (2) from theorem \ref{chandra-merlin} and (3) is proposition \ref{structure-functor-bijection}.
\end{proof}

\section{Relational Semantics and Natural Language Entailment}
%!TEX root = discocat-act2019.tex
Relational models are rigid monoidal functors $F : \bf{G} \rightarrow \bf{Rel}$ for $\bf{G}$ the rigid monoidal category generated by a pregroup grammar $G = (V, B, \Delta)$.
We require that the image for words $w \in V$ be the singleton $F(w) = 1$, hence the image for a dictionnary entry $(w, t) \in \Delta$ is given by a subset $F(w \to t) \sub F(t)$.
Without loss of generality, we assume that the image of $F$ lies in $\bf{Rel}\vert_U$ for some finite universe $U$, which may be taken to be the union of the universe for each basic type $U = \bigcup_{b \in B} F(b)$.

\begin{lemma}\label{lemma-rel-models}
  A relational model $F : \bf{G} \to \bf{Rel}\vert_U$ is uniquely determined by its image on basic types and on dictionnary entries, i.e. by a function $\tt{ar} : B \to \N$ and a subset $F(w \to t) \sub U^{\tt{ar}(t)}$ for each $(w, t) \in \Delta$.
  Thus, $F$ induces a structure $K \in \cal{M}_\Delta(U)$ over the dictionnary seen as a signature where entries $(w, t) \in \Delta$ are symbols of arity $F(t) \in \N$.
\end{lemma}

\begin{proof}
  This follows from lemma~\ref{get-rid-of-induced-steps} and the universal property of the free rigid category: the functor $F : \bf{G} \to \bf{Rel}\vert_U$ is uniquely defined by its image on generators, i.e. on the basic types $b \in B$ and the dictionnary entries $(w, t) \in \Delta$ seen as generating arrows $w \to t$.
\end{proof}

Note that the data defining a relational model is a finite set of triples
$\set{w : t :: F(w \to t)}_{(w, t) \in \Delta}$, called a \emph{pregroup lexicon} in \cite{Preller14}.
This allows us to define $\tt{Semantics}$ as a function problem with relational models as input.

\begin{definition}\label{def-semantics}
    \begin{problem}
    \problemtitle{$\tt{Semantics}$}
    \probleminput{$r \in \bf{G}(u, s), \quad F : \bf{G} \to \bf{Rel}\vert_U$}
    \problemoutput{$F(r) \sub U^{F(s)}$}
  \end{problem}
\end{definition}

\begin{lemma}\label{lemma-factorisation}
  Every relational model $F : \bf{G} \to \bf{Rel}\vert_U$ factorises as $F = K \circ L$ for a relational structure $K \in \cal{M}_\Delta(U)$ and a rigid monoidal functor $L : \bf{G} \to \bf{CB}(\Delta)$.
\end{lemma}

\begin{proof}
  By propositions \ref{structure-functor-bijection} and \ref{lemma-rel-models}, $F : \bf{G} \to \bf{Rel}\vert_U$ induces a morphism of Cartesian bicategories $K : \bf{CB}(\Delta) \to \bf{Rel}\vert_U$.
  The functor $L : \bf{G} \to \bf{CB}(\Delta)$ sends each dictionnary entry to itself as a relational symbol, by construction we have $K \circ L = F$.
\end{proof}

\begin{proposition}
  There is a logspace reduction from $\tt{Semantics}$ to conjunctive query $\tt{Evaluation}$, hence $\tt{Semantics} \in \tt{NP}$.
\end{proposition}

\begin{proof}
  The factorisation $K \circ L = F$ of lemma~\ref{lemma-factorisation} and the translation $\Lambda$ of theorem~\ref{translation} are in logspace, they give a query $\phi = \Lambda (L(r)) \in \cal{Q}_\Delta$ such that $\tt{eval}(\phi, K) = F(r)$.
\end{proof}

We conjecture that the constraint language induced by a pregroup grammar meets the tractability condition for the CSP dichotomy theorem \cite{Bulatov17}.

\begin{conjecture}
  Fix $G = (V, B, \Delta)$, $\tt{Semantics}$ is poly-time computable in the size of $(u, s) \in List(V) \times P(B)$ and in the size of the universe $U$.
\end{conjecture}

We can also consider \emph{free} models, i.e. rigid monoidal functors $L : \bf{G} \to \bf{C}$ for a finitely presented Cartesian bicategory $\bf{C}$.
For example, take $\bf{C}$ to be generated by the signature $\Sigma = \set{\text{Leib, Spin, phil, \dots}}$ as 1-arrows with codomain given by the function $\tt{ar} : \Sigma \to \N$ and the following set of 2-arrows:
%!TEX root = ../discocat-act2019.tex
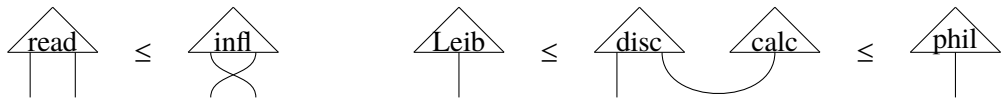
\begin{equation*}
  \begin{tikzpicture}[scale=0.6]
  		\node (0) at (-3, 0) {};
  		\node (1) at (-1, 0) {};
  		\node (2) at (-2, 1) {};
  		\node (3) at (0, 0) {};
  		\node (4) at (1, 1) {};
  		\node (5) at (2, 0) {};
  		\node (6) at (4, 0) {};
  		\node (7) at (5, 1) {};
  		\node (8) at (6, 0) {};
  		\node (9) at (5, 0.25) {phil};
  		\node (10) at (1, 0.25) {calc};
  		\node (11) at (-2, 0.25) {disc};
  		\node (12) at (-2.5, 0) {};
  		\node (13) at (-1.5, 0) {};
  		\node (14) at (1, 0) {};
  		\node (15) at (5, 0) {};
  		\node (16) at (-2.5, -1) {};
  		\node (17) at (5, -1) {};
  		\node (18) at (3, 0) {$\leq$};
  		\node (19) at (-5, 0) {};
  		\node (20) at (-6, 1) {};
  		\node (21) at (-7, 0) {};
  		\node (22) at (-6, 0) {};
  		\node (23) at (-6, -1) {};
  		\node (24) at (-6, 0.25) {Leib};
  		\node (25) at (-4, 0) {$\leq$};
  		\node (26) at (-10, 0) {};
  		\node (27) at (-11, 1) {};
  		\node (28) at (-12, 0) {};
  		\node (29) at (-11.5, 0) {};
  		\node (30) at (-10.5, 0) {};
  		\node (31) at (-10.5, -1) {};
  		\node (32) at (-11.5, -1) {};
  		\node (33) at (-15, 1) {};
  		\node (34) at (-14, 0) {};
  		\node (35) at (-15.5, 0) {};
  		\node (36) at (-16, 0) {};
  		\node (37) at (-15.5, -1) {};
  		\node (38) at (-14.5, 0) {};
  		\node (39) at (-14.5, -1) {};
  		\node (40) at (-11, 0.25) {infl};
  		\node (41) at (-15, 0.25) {read};
  		\node (42) at (-13, 0) {$\leq$};

  		\draw [bend right=90, looseness=1.25] (13.center) to (14.center);
  		\draw (12.center) to (16.center);
  		\draw (15.center) to (17.center);
  		\draw (0.center) to (2.center);
  		\draw (2.center) to (1.center);
  		\draw (1.center) to (0.center);
  		\draw (3.center) to (4.center);
  		\draw (4.center) to (5.center);
  		\draw (5.center) to (3.center);
  		\draw (6.center) to (7.center);
  		\draw (7.center) to (8.center);
  		\draw (8.center) to (6.center);
  		\draw (21.center) to (20.center);
  		\draw (20.center) to (19.center);
  		\draw (19.center) to (21.center);
  		\draw (22.center) to (23.center);
  		\draw (28.center) to (27.center);
  		\draw (27.center) to (26.center);
  		\draw (26.center) to (28.center);
  		\draw (34.center) to (33.center);
  		\draw (33.center) to (36.center);
  		\draw (36.center) to (34.center);
  		\draw (35.center) to (37.center);
  		\draw (38.center) to (39.center);
  		\draw [in=90, out=-90] (29.center) to (31.center);
  		\draw [in=-90, out=90] (32.center) to (30.center);

  \end{tikzpicture}
\end{equation*}
The 2-arrows of $\bf{C}$ encode \emph{existential rules} of the form $\forall \ x_0 \ \cdots \ \forall \ x_k \ \cdot \ \phi \to \phi'$ for two conjunctive queries $\phi, \phi'$ with $\tt{fv}(\phi) = \tt{fv}(\phi') = \set{x_0, \dots, x_k}$, also called tuple-generating dependencies in database theory, see \cite{Thomazo13} for a survey.
The composition of 2-arrows in $\bf{C}$ then allow us to compute entailment, e.g.:
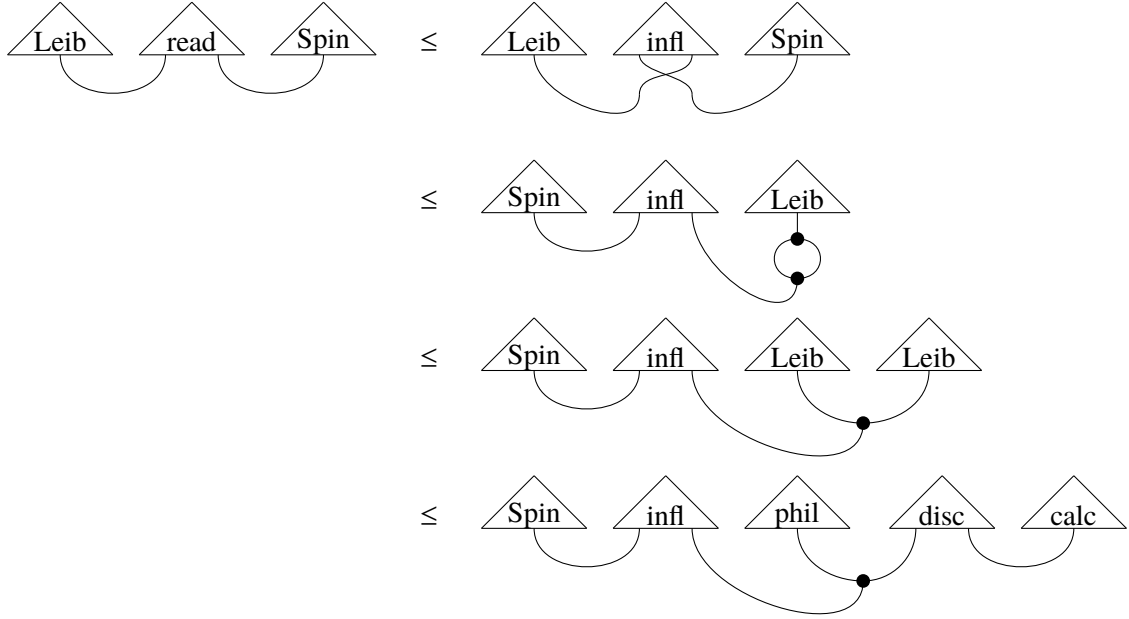
\begin{equation*}
\begin{tikzpicture}[scale=0.7]
  		\node (0) at (-1, 0) {};
  		\node (1) at (-2, 1) {};
  		\node (2) at (-3, 0) {};
  		\node (3) at (-2, 0) {};
  		\node (4) at (-3.5, 0) {};
  		\node (5) at (-4.5, 1) {};
  		\node (6) at (-5.5, 0) {};
  		\node (7) at (-5, 0) {};
  		\node (8) at (-4, 0) {};
  		\node (9) at (-6, 0) {};
  		\node (10) at (-7, 1) {};
  		\node (11) at (-8, 0) {};
  		\node (12) at (-7, 0) {};
  		\node (13) at (8, 0) {};
  		\node (14) at (7, 1) {};
  		\node (15) at (6, 0) {};
  		\node (16) at (7, 0) {};
  		\node (17) at (5.5, 0) {};
  		\node (18) at (4.5, 1) {};
  		\node (19) at (3.5, 0) {};
  		\node (20) at (4, 0) {};
  		\node (21) at (5, 0) {};
  		\node (22) at (3, 0) {};
  		\node (23) at (2, 1) {};
  		\node (24) at (1, 0) {};
  		\node (25) at (2, 0) {};
  		\node (26) at (-7, 0.25) {Leib};
  		\node (27) at (-4.5, 0.25) {read};
  		\node (28) at (-2, 0.25) {Spin};
  		\node (29) at (2, 0.25) {Leib};
  		\node (30) at (4.5, 0.25) {infl};
  		\node (31) at (7, 0.25) {Spin};
  		\node (32) at (0, 0.25) {$\leq$};
  		\node (33) at (5, -0.75) {};
  		\node (34) at (4, -0.75) {};
  		\node (35) at (8, -3) {};
  		\node (36) at (7, -2) {};
  		\node (37) at (6, -3) {};
  		\node (38) at (7, -3) {};
  		\node (39) at (5.5, -3) {};
  		\node (40) at (4.5, -2) {};
  		\node (41) at (3.5, -3) {};
  		\node (42) at (4, -3) {};
  		\node (43) at (5, -3) {};
  		\node (44) at (3, -3) {};
  		\node (45) at (2, -2) {};
  		\node (46) at (1, -3) {};
  		\node (47) at (2, -3) {};
  		\node (48) at (2, -2.75) {Spin};
  		\node (49) at (4.5, -2.75) {infl};
  		\node (50) at (7, -2.75) {Leib};
  		\node (51) at (0, -2.75) {$\leq$};
  		\node[scale=0.5, circle, fill=black] (52) at (7, -3.5) {};
  		\node[scale=0.5, circle, fill=black] (53) at (7, -4.25) {};
  		\node (66) at (0, -5.75) {$\leq$};
  		\node (67) at (8, -6) {};
  		\node (68) at (7, -5) {};
  		\node (69) at (6, -6) {};
  		\node (70) at (7, -6) {};
  		\node (71) at (5.5, -6) {};
  		\node (72) at (4.5, -5) {};
  		\node (73) at (3.5, -6) {};
  		\node (74) at (4, -6) {};
  		\node (75) at (5, -6) {};
  		\node (76) at (3, -6) {};
  		\node (77) at (2, -5) {};
  		\node (78) at (1, -6) {};
  		\node (79) at (2, -6) {};
  		\node (80) at (2, -5.75) {Spin};
  		\node (81) at (4.5, -5.75) {infl};
  		\node (82) at (7, -5.75) {Leib};
  		\node[scale=0.5, circle, fill=black] (84) at (8.25, -7) {};
  		\node (85) at (10.5, -6) {};
  		\node (86) at (9.5, -5) {};
  		\node (87) at (8.5, -6) {};
  		\node (88) at (9.5, -6) {};
  		\node (89) at (9.5, -5.75) {Leib};
  		\node (90) at (8, -9) {};
  		\node (91) at (7, -8) {};
  		\node (92) at (6, -9) {};
  		\node (93) at (7, -9) {};
  		\node (94) at (5.5, -9) {};
  		\node (95) at (4.5, -8) {};
  		\node (96) at (3.5, -9) {};
  		\node (97) at (4, -9) {};
  		\node (98) at (5, -9) {};
  		\node (99) at (3, -9) {};
  		\node (100) at (2, -8) {};
  		\node (101) at (1, -9) {};
  		\node (102) at (2, -9) {};
  		\node (103) at (2, -8.75) {Spin};
  		\node (104) at (4.5, -8.75) {infl};
  		\node (105) at (7, -8.75) {phil};
  		\node[scale=0.5, circle, fill=black] (106) at (8.25, -10) {};
  		\node (107) at (10.75, -9) {};
  		\node (108) at (9.75, -8) {};
  		\node (109) at (8.75, -9) {};
  		\node (110) at (9.25, -9) {};
  		\node (111) at (9.75, -8.75) {disc};
  		\node (112) at (0, -8.75) {};
  		\node (113) at (0, -8.75) {$\leq$};
  		\node (114) at (10.25, -9) {};
  		\node (115) at (11.25, -9) {};
  		\node (116) at (12.25, -8) {};
  		\node (117) at (13.25, -9) {};
  		\node (118) at (12.25, -9) {};
  		\node (119) at (12.25, -8.75) {};
  		\node (120) at (12.25, -8.75) {calc};

  		\draw (11.center) to (10.center);
  		\draw (10.center) to (9.center);
  		\draw (9.center) to (11.center);
  		\draw (6.center) to (5.center);
  		\draw [bend right=90, looseness=1.25] (12.center) to (7.center);
  		\draw [bend left=90, looseness=1.25] (3.center) to (8.center);
  		\draw (6.center) to (4.center);
  		\draw (4.center) to (5.center);
  		\draw (2.center) to (1.center);
  		\draw (1.center) to (0.center);
  		\draw (0.center) to (2.center);
  		\draw (24.center) to (23.center);
  		\draw (23.center) to (22.center);
  		\draw (22.center) to (24.center);
  		\draw (19.center) to (18.center);
  		\draw (19.center) to (17.center);
  		\draw (17.center) to (18.center);
  		\draw (15.center) to (14.center);
  		\draw (14.center) to (13.center);
  		\draw (13.center) to (15.center);
  		\draw [in=90, out=-90, looseness=0.75] (20.center) to (33.center);
  		\draw [in=90, out=-90] (21.center) to (34.center);
  		\draw [in=-90, out=-90] (25.center) to (34.center);
  		\draw [in=-90, out=-90] (33.center) to (16.center);
  		\draw (46.center) to (45.center);
  		\draw (45.center) to (44.center);
  		\draw (44.center) to (46.center);
  		\draw (41.center) to (40.center);
  		\draw [bend right=90, looseness=1.25] (47.center) to (42.center);
  		\draw (41.center) to (39.center);
  		\draw (39.center) to (40.center);
  		\draw (37.center) to (36.center);
  		\draw (36.center) to (35.center);
  		\draw (35.center) to (37.center);
  		\draw [bend right=90, looseness=2.00] (53.center) to (52.center);
  		\draw (52.center) to (38.center);
  		\draw [in=-90, out=-90, looseness=1.25] (43.center) to (53.center);
  		\draw [bend right=90, looseness=2.00] (52.center) to (53.center);
  		\draw (78.center) to (77.center);
  		\draw (77.center) to (76.center);
  		\draw (76.center) to (78.center);
  		\draw (73.center) to (72.center);
  		\draw [bend right=90, looseness=1.25] (79.center) to (74.center);
  		\draw (73.center) to (71.center);
  		\draw (71.center) to (72.center);
  		\draw (69.center) to (68.center);
  		\draw (68.center) to (67.center);
  		\draw (67.center) to (69.center);
  		\draw [in=-90, out=-90] (75.center) to (84.center);
  		\draw (87.center) to (86.center);
  		\draw (86.center) to (85.center);
  		\draw (85.center) to (87.center);
  		\draw [in=-90, out=0] (84.center) to (88.center);
  		\draw [in=180, out=-90] (70.center) to (84.center);
  		\draw (101.center) to (100.center);
  		\draw (100.center) to (99.center);
  		\draw (99.center) to (101.center);
  		\draw (96.center) to (95.center);
  		\draw [bend right=90, looseness=1.25] (102.center) to (97.center);
  		\draw (96.center) to (94.center);
  		\draw (94.center) to (95.center);
  		\draw (92.center) to (91.center);
  		\draw (91.center) to (90.center);
  		\draw (90.center) to (92.center);
  		\draw [in=-90, out=-90] (98.center) to (106.center);
  		\draw (109.center) to (108.center);
  		\draw (108.center) to (107.center);
  		\draw (107.center) to (109.center);
  		\draw [in=-90, out=0] (106.center) to (110.center);
  		\draw [in=180, out=-90] (93.center) to (106.center);
  		\draw (115.center) to (116.center);
  		\draw (116.center) to (117.center);
  		\draw (117.center) to (115.center);
  		\draw [bend right=90, looseness=1.25] (114.center) to (118.center);
\end{tikzpicture}
\end{equation*}
where the second and third inequations follow from the axioms of definition~\ref{def-CB} (adjointness and lax comonoidality), the first and last from the generators.

\begin{definition}
    \begin{problem}
    \problemtitle{$\tt{Entailment}$}
    \probleminput{$r \in \bf{G}(u, s), \quad r' \in \bf{G}(u', s), \quad L : \bf{G} \to \bf{C}$}
    \problemoutput{$L(r) \leq L(r')$}
  \end{problem}
\end{definition}

\begin{proposition}
  $\tt{Entailment}$ is undecidable for finitely presented Cartesian bicategories.
  When $\bf{C}$ is freely generated, the problem reduces to conjunctive query $\tt{Containment}$.
\end{proposition}

\begin{proof}
Entailment of conjunctive queries under existential rules is undecidable, see \cite{BagetMugnier02}.
When $\bf{C} = \bf{CB}(\Sigma)$ is freely generated by a relational signature $\Sigma$, i.e. with no existential rules, theorem~\ref{translation} yields a logspace reduction to $\tt{Containment}$: $\tt{Entailment} \in \tt{NP}$.
\end{proof}

Models in free Cartesian bicategories make $\tt{Entailment}$ a decidable problem, they also allow us to reformulate lemma~\ref{lemma-factorisation} as follows:
every model $F : \bf{G} \to \bf{Rel}$ factorises as $F = K \circ L$ for a free model $L : \bf{G} \to \bf{CB}(\Sigma)$ and a relational structure $K : \bf{CB}(\Sigma) \to \bf{Rel}$.
In the next section, we use this fact to define  $\tt{QuestionAnswering}$ as a computational problem.

\section{Question Answering as an NP-complete Problem}
%!TEX root = discocat-act2019.tex
We consider the following computational problem: given a natural language corpus and a question, does the corpus contain an answer?
We show how to translate a corpus into a relational database so that question answering reduces to query evaluation.
We fix a free model $L : \bf{G} \to \bf{CB}(\Sigma)$ with $L(s) = 0$, i.e. grammatical sentences are mapped to closed formulae.
We assume that $L(q) = L(a)$ for $q$ and $a$ the question and answer types respectively, i.e. both are mapped to queries with the same number of free variable. Lexical items such as ``influence'' and ``Leibniz'' are mapped to their own symbol in the relational signature $\Sigma$, whereas functional words such as relative pronouns are sent to the Frobenius algebra of $\bf{CB}(\Sigma)$, see \cite{SadrzadehEtAl13}.

We define a corpus as a set of sentences $u \in List(V)$ with parsing $r : u \to s$ in $\bf{G}$, i.e. a subset $C \sub \coprod_{u \in List(V)} \bf{G}(u, s)$.
If we apply $L$ independently to each sentence, the resulting queries have disjoint sets of variables.
In order to obtain the desired database, we need to map variables to some designated entities: a standard natural language processing task called \emph{entity linking} (EL).

Let $\phi_C = \bigwedge_{r \in C} \Lambda(L(r))$ be the conjunction of each sentence in the corpus, where $\Lambda$ is the translation from diagrams to conjunctive queries of theorem~\ref{translation}.
We define an entity linking for $C$ as a function $\mu : \tt{var}(\phi_C) \to E$ for some finite set $E$ of entities.
Thus, we get the following algorithm for translating the corpus $C$ with entity linking $\mu : \tt{var}(\phi_C) \to E$ into a $\Sigma$-structure:
\begin{enumerate}
  \item translate each parsed sentence $r \in C$ into a conjunctive query $\Lambda(L(r)) \in \cal{Q}_\Delta$,
  \item compute their conjunction $\phi_C = \bigwedge_{r \in C} \Lambda(L(r))$ and the substitution $\mu(\phi_C)$,
  \item construct the corresponding canonical structure $K = CM(\mu(\phi_C)) \in \cal{M}_\Sigma(E)$.
\end{enumerate}

\begin{definition}
    \begin{problem}
    \problemtitle{$\tt{QuestionAnswering}$}
    \probleminput{$C \sub \coprod_{u \in List(V)} \bf{G}(u, s), \quad \mu : \tt{var}(\phi_C) \to E, \quad \phi \in \cal{Q}_\Sigma$}
    \problemoutput{$\tt{eval}(\phi, K) \sub E^{\tt{fv}(\phi)} \quad$ where $K = CM(\mu(\phi_C))$}
  \end{problem}
\end{definition}

\begin{theorem}
  $\tt{QuestionAnswering}$ is $\tt{NP-complete}$.
\end{theorem}

\begin{proof}
  Membership follows immediately by reduction to $\tt{Evaluation}$.
  Hardness follows by reduction from graph homomorphism, we only give a sketch of proof and refer to \cite{CoeckeEtAl18a} where EL is called matching.
  Any graph can be encoded in a corpus given by a set of subject-verb-object sentences, where EL maps nouns to their corresponding node.
\end{proof}

\begin{example}
  We take $L : \bf{G} \to \bf{CB}(\Sigma)$ to map the question word ``Who'' to the compact-closed structure, the determinant ``a'' to the unit and the common noun ``philosopher'' to the symbol $phil \in \Sigma$ composed with the comonoid.
  We can now find the nouns that answer the question $r \in \bf{G}(u, q)$ of example~\ref{example} as the evaluation of the following query:
  \begin{equation*}\begin{aligned}
L(r) \s &= \s
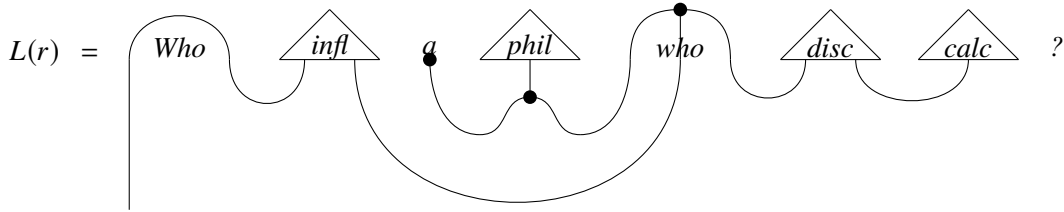
\begin{tikzpicture}[scale=0.666, baseline=(O.base)]
    \node (O) at (0, 0) {};
		\node (0) at (6, 0) {};
		\node (1) at (5, 1) {};
		\node (2) at (4, 0) {};
		\node (3) at (5, 0) {};
		\node (4) at (2, 0) {};
		\node (5) at (1, 1) {};
		\node (6) at (0, 0) {};
		\node (7) at (0.5, 0) {};
		\node (8) at (1.5, 0) {};
		\node (10) at (1, 0.25) {infl};
		\node (11) at (5, 0.25) {phil};
		\node [scale=0.5, circle, fill=black] (12) at (8, 1) {};
		\node (13) at (12, 0) {};
		\node (14) at (11, 1) {};
		\node (15) at (10, 0) {};
		\node (16) at (10.5, 0) {};
		\node (17) at (11, 0.25) {disc};
		\node (18) at (11.5, 0) {};
		\node (19) at (12.75, 0) {};
		\node (20) at (13.75, 1) {};
		\node (21) at (14.75, 0) {};
		\node (22) at (13.75, 0) {};
		\node (23) at (13.75, 0.25) {};
		\node (24) at (13.75, 0.25) {calc};
		\node [scale=0.5, circle, fill=black] (26) at (5, -0.75) {};
		\node (27) at (3, 0.25) {a};
		\node (28) at (8, 0.25) {who};
		\node (29) at (-1, 0) {};
		\node (30) at (-3, 0) {};
		\node (31) at (-2, 0.25) {Who};
		\node (32) at (15.5, 0.25) {};
		\node (33) at (15.5, 0.25) {?};
		\node (34) at (-3, -3) {};
		\node [scale=0.5, circle, fill=black] (37) at (3, 0) {};
		\node (38) at (7, 0) {};
		\node (39) at (8, 0) {};
		\node (40) at (9, 0) {};
		\node (41) at (4, -1.5) {};
		\node (42) at (6, -1.5) {};

		\draw (6.center) to (5.center);
		\draw (6.center) to (4.center);
		\draw (4.center) to (5.center);
		\draw (2.center) to (1.center);
		\draw (1.center) to (0.center);
		\draw (0.center) to (2.center);
		\draw (15.center) to (14.center);
		\draw (14.center) to (13.center);
		\draw (13.center) to (15.center);
		\draw (19.center) to (20.center);
		\draw (20.center) to (21.center);
		\draw (21.center) to (19.center);
		\draw [bend right=90, looseness=1.25] (18.center) to (22.center);
		\draw (3.center) to (26.center);
		\draw [bend left=90, looseness=2.00] (7.center) to (29.center);
		\draw [bend left=90, looseness=1.50] (30.center) to (29.center);
		\draw (30.center) to (34.center);
		\draw [in=0, out=180, looseness=1.25] (26.center) to (41.center);
		\draw [in=180, out=0, looseness=1.25] (26.center) to (42.center);
		\draw [in=-90, out=-180] (41.center) to (37.center);
		\draw [in=-90, out=0] (42.center) to (38.center);
		\draw [in=-180, out=90, looseness=1.25] (38.center) to (12.center);
		\draw (12.center) to (39.center);
		\draw [in=90, out=0, looseness=1.25] (12.center) to (40.center);
		\draw [bend right=90, looseness=1.75] (40.center) to (16.center);
		\draw [bend left=90, looseness=1.50] (39.center) to (8.center);
\end{tikzpicture}
\\
\Lambda(L(r)) \s &= \s \exists \ x_1 \ \exists \ x_2 \ \cdot \ infl(x_0, x_1) \s \land \s phil(x_1) \s \land \s disc(x_1, x_2) \s \land \s calc(x_2)
\end{aligned}\end{equation*}

  If ``Spinoza influenced the philosopher Leibniz'' and ``Leibniz discovered calculus'' are in the corpus $C$, we have $L(Spinoza \to n) \in \tt{QuestionAnswering}(C, \mu, \Lambda (L(r)))$.
\end{example}

\section*{Conclusion}
We showed that DisCo models in the category of sets and relations factorise
through cartesian bicategories. As these provide a categorical formulation
of conjunctive queries and relational databases,
we were able to lift computational problems and complexity-theoretic results
from database theory to natural language processing.
This opens up many avenues for future work:
\begin{itemize}
  \item text summarisation through conjunctive query minimisation \cite{ChekuriRajaraman00},
  \item semantics of ``How many?'' questions and counting problems \cite{StefanoniEtAl18},
  \item many-sorted relational models with graphical regular logic \cite{FongSpivak18a},
  \item from Boolean semantics to generalised relations in arbitrairy topoi \cite{CoeckeEtAl18},
  \item from regular logic to description logics in bicategories of relations \cite{Patterson17},
  \item comonadic semantics for bounded short-term memory \cite{AbramskyShah18},
  \item quantum speedup for question answering via Grover's search \cite{ZengCoecke16}.
\end{itemize}

\section*{Acknowledgments}

The authors would like to thank Bob Coecke, Dan Marsden, Antonin Delpeuch, Vincent Wang, Jacob Leygonie, Dusko Pavlovic, Rui Soares Barbosa and Samson Abramsky for inspiring discussions
in the process of writing this article,
as well as anonymous Reviewers of ACT2019 for constructive comments that improved the presentation of this work.
K.M. is supported by the EPSRC NQIT Hub and Cambridge Quantum Computing Ltd.
A.T. acknowledges Simon Harrison for financial DPhil support.

\bibliographystyle{eptcs}
\bibliography{discocat-complexity}

\end{document}